\documentclass[nohyperref]{article}

\usepackage{microtype}
\usepackage{graphicx}
\usepackage{subfigure}
\usepackage{booktabs} 

\usepackage{hyperref}

\usepackage[accepted]{icml2022}

\usepackage{amssymb}
\usepackage{amsmath}
\usepackage{mathtools}
\usepackage{bbding}
\usepackage{ulem}
\usepackage{amsthm}
\usepackage{braket}
\usepackage{enumerate}
\usepackage{varwidth}

\usepackage[capitalize,noabbrev]{cleveref}

\theoremstyle{plain}
\newtheorem{theorem}{Theorem}[section]
\newtheorem{claim}[theorem]{Claim}

\newtheorem{lemma}[theorem]{Lemma}

\theoremstyle{definition}
\newtheorem{definition}[theorem]{Definition}

\theoremstyle{remark}

\usepackage{multirow}
\usepackage{tikz}
\usepackage{pgfplots}
\usepackage{pgfplotstable}
\usepackage{makecell}

\useunder{\uline}{\ul}{}

\usepackage[textsize=tiny]{todonotes}
\icmltitlerunning{HousE: Knowledge Graph Embedding with Householder Parameterization}

\begin{document}

\twocolumn[
\icmltitle{HousE: Knowledge Graph Embedding with Householder Parameterization}



\icmlsetsymbol{equal}{*}
\icmlsetsymbol{in}{\dag}

\begin{icmlauthorlist}
\icmlauthor{Rui Li}{dut,equal}
\icmlauthor{Jianan Zhao}{ndu}
\icmlauthor{Chaozhuo Li}{comp}
\icmlauthor{Di He}{pku}
\icmlauthor{Yiqi Wang}{msu}
\icmlauthor{Yuming Liu}{ms} \\
\icmlauthor{Hao Sun}{ms}
\icmlauthor{Senzhang Wang}{csu}
\icmlauthor{Weiwei Deng}{ms}
\icmlauthor{Yanming Shen}{dut}
\icmlauthor{Xing Xie}{comp}
\icmlauthor{Qi Zhang}{ms}
\end{icmlauthorlist}

\icmlaffiliation{dut}{Department of Computer Science and Technology, Dalian University of Technology, Dalian, China}
\icmlaffiliation{comp}{Microsoft Research Asia, Beijing, China}
\icmlaffiliation{ndu}{University of Notre Dame, Indiana, USA}
\icmlaffiliation{pku}{Peking University, Beijing, China}
\icmlaffiliation{msu}{Michigan State University, Michigan, USA}
\icmlaffiliation{csu}{Central South University, Changsha, China}
\icmlaffiliation{ms}{Microsoft, Beijing, China}

\icmlcorrespondingauthor{Yanming Shen}{shen@dlut.edu.cn}
\icmlcorrespondingauthor{Chaozhuo Li}{cli@microsoft.com}

\icmlkeywords{Machine Learning, ICML}

\vskip 0.3in
]



\printAffiliationsAndNotice{\icmlEqualContribution}
\begin{abstract} 
The effectiveness of knowledge graph embedding (KGE) largely depends on the ability to model intrinsic relation patterns and mapping properties. 
However, existing approaches can only capture some of them with insufficient modeling capacity. 
In this work, we propose a more powerful KGE framework named HousE, which involves a novel parameterization based
on two kinds of Householder transformations: (1) \normalem{\emph{Householder rotations}} to achieve superior capacity of modeling relation patterns; (2) \normalem{\emph{Householder projections}} to handle sophisticated relation mapping properties. 
Theoretically, HousE is capable of modeling crucial relation patterns and mapping properties simultaneously. Besides, HousE is a generalization of existing rotation-based models while extending the rotations to high-dimensional spaces. Empirically, HousE achieves new state-of-the-art performance on five benchmark datasets. Our code is available at \url{https://github.com/anrep/HousE}. 
\end{abstract}

\section{Introduction}
\label{submission}

Knowledge graphs (KGs) store massive human knowledge as a collection of factual triples, where each triple $(h,r,t)$ represents a relation $r$ between head entity $h$ and tail entity $t$. 
With a wealth of human knowledge, KGs have demonstrated their effectiveness in a myriad of downstream applications \cite{xiong2017explicit}.
However, real-world KGs such as Freebase \cite{bollacker2008freebase} and Yago \cite{suchanek2007yago}) usually suffer from incompleteness. 
Knowledge Graph Embedding (KGE), which learns low-dimensional representations for entities and relations, excels as an effective tool for predicting missing links.
A crucial challenge of KGE lies in 
how to model relation patterns (e.g., symmetry, antisymmetry, inversion and composition) and relation mapping properties (RMPs, i.e., 1-to-1, 1-to-N, N-to-1 and N-to-N)~\cite{TransE, sun2019rotate}  as shown in Figure \ref{Fig-1}. 
Most works design specific vector spaces and operations to capture such patterns and RMPs. For example, 
TransE \cite{TransE} represents relations as translations, which fails in modeling symmetry and RMPs. 
Recently, RotatE \cite{sun2019rotate} represents relations as rotations in the complex plane to model the four relation patterns, but it is incapable of handling RMPs due to the distance-preserving property of rotations.
Rotate3D \cite{gao2020rotate3d} and QuatE \cite{zhang2019quaternion} introduce quaternions to extend rotations to 3-dimensional and 4-dimensional spaces, and achieve
better performance with larger model capacity. 


\begin{figure}[t!]
\centering
\subfigure[Symmetry]{\includegraphics[width=0.3\columnwidth]{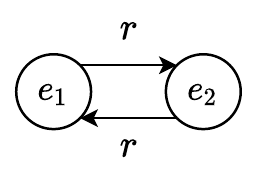}}
\subfigure[Antisymmetry]{\includegraphics[width=0.3\columnwidth]{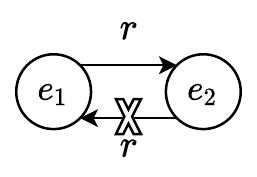}}
\subfigure[Inversion]{\includegraphics[width=0.3\columnwidth]{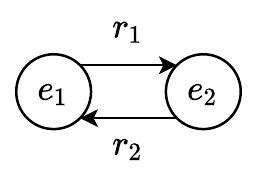}}
\subfigure[Composition]{\includegraphics[width=0.3\columnwidth]{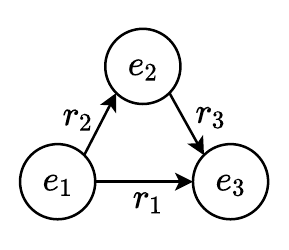}}
\subfigure[1-to-N]{\includegraphics[width=0.3\columnwidth]{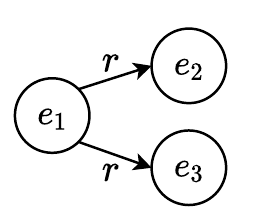}}
\subfigure[N-to-1]{\includegraphics[width=0.3\columnwidth]{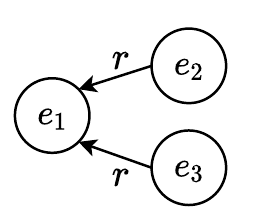}}
\vspace{-2mm}
\caption{Illustrations of four relation patterns (a-d)~\cite{sun2019rotate} and two challenging RMPs (e-f)~\cite{TransE}.}
\label{Fig-1}
\vskip -0.2in
\end{figure}

\begin{table*}[t]
\caption{Recent models' capability of modeling relation patterns and relation mapping properties (RMPs). TransX represents many TransE’s variants, such as TransH~\cite{wang2014knowledge}, TransR~\cite{lin2015learning} and TransD~\cite{ji2015knowledge}.}
\label{ability}
\begin{center}
\begin{small}
\vskip 0.05in
\resizebox{0.75\textwidth}{!}{
\begin{tabular}{lccccccc}
\toprule
Model & Symmetry & Antisymmetry & Inversion & Composition & RMPs & \makecell[c]{Dim. of Rotation} \\
\midrule
TransE    & \XSolidBrush & \CheckmarkBold & \CheckmarkBold & \CheckmarkBold &  \XSolidBrush & - \\
TransX    & \CheckmarkBold & \CheckmarkBold &  \XSolidBrush &  \XSolidBrush  & \CheckmarkBold & - \\
DistMult    & \CheckmarkBold & \XSolidBrush & \XSolidBrush & \XSolidBrush  & \CheckmarkBold & - \\
ComplEx    & \CheckmarkBold & \CheckmarkBold & \CheckmarkBold & \XSolidBrush  & \CheckmarkBold & - \\
RotatE    & \CheckmarkBold & \CheckmarkBold & \CheckmarkBold & \CheckmarkBold  & \XSolidBrush & 2-D \\
Rotate3D    & \CheckmarkBold & \CheckmarkBold & \CheckmarkBold & \CheckmarkBold  & \XSolidBrush & 3-D \\
QuatE    & \CheckmarkBold & \CheckmarkBold & \CheckmarkBold & \XSolidBrush  & \CheckmarkBold & 4-D \\
DualE    & \CheckmarkBold & \CheckmarkBold & \CheckmarkBold & \XSolidBrush  & \CheckmarkBold & 3-D \\
\midrule
HousE-r    & \CheckmarkBold & \CheckmarkBold & \CheckmarkBold & \CheckmarkBold  & \XSolidBrush & $k$-D \\
HousE    & \CheckmarkBold & \CheckmarkBold & \CheckmarkBold & \CheckmarkBold  & \CheckmarkBold & $k$-D \\
\bottomrule
\end{tabular}}
\end{small}
\end{center}
\vspace{-4mm}
\end{table*}

However, as far as we know, none of the existing methods is capable of modeling all the relation patterns and RMPs as shown in Table \ref{ability}, leading to the sub-optimal performance. 
Furthermore, some advanced approaches, such as \cite{sun2019rotate,gao2020rotate3d,zhang2019quaternion}, are specifically designed on 2,3,4 dimensional spaces, which may be inadequate to capture the sophisticated structures of KGs  \cite{zhang2019quaternion}. 
Therefore, this brings us a question: \normalem{\emph{is there a framework to handle all the above relation patterns and RMPs with more powerful modeling capacity?}}

In this paper, we give an affirmative answer by proposing a more powerful and general framework named HousE based on Householder parameterization. 
We prove that the composition of $2\lfloor \frac{k}{2} \rfloor$ Householder reflections \cite{householder1958unitary} can represent any $k$-dimensional rotations. 
This unique property of Householder reflections provides us a natural way to model high-dimensional rotations with more degree of freedom.
We call this kind of rotations as \normalem{\emph{Householder rotations}}, based on which a simple model named HousE-r is proposed to achieve superior capacity of modeling relation patterns.
Nevertheless, HousE-r is plagued by the sophisticated RMPs due to the distance-preserving nature of pure Householder rotations.
To remedy this deficiency, we modify the vanilla Householder reflections to \normalem{\emph{Householder projections}}, which can flexibly adjust the relative distances between points. 
The Householder projections are further integrated with HousE-r to establish the final HousE.
By enjoying the merits of Householder rotations and Householder projections,
HousE is theoretically capable of modeling all the relation patterns and RMPs shown in Table \ref{ability}.   
Moreover, our proposal is a general framework and existing rotation-based models can be viewed as the special cases of HousE.  
Our contributions are summarized as follows:
\begin{itemize}
\item To the best of our knowledge, we are the first to utilize Householder parameterization  to build a more powerful and general KGE framework named HousE. 
\item We present a simple way to represent relations as high-dimensional Householder rotations, which endows HousE with better modeling capacity.

\item We further modify the vanilla Householder reflections to Householder projections. By combining the Householder projections and rotations, HousE is able to model all the relation patterns and RMPs in Table \ref{ability}.

\item We conduct extensive experiments over five benchmarks and our proposal consistently outperforms SOTA baselines over all the datasets. 
\end{itemize}

\section{Problem Setup}



Given the entity set $\mathcal{E}$ and relation set $\mathcal{R}$, a knowledge graph can be formally defined as a collection of factual triples $\mathcal{D}=\{(h,r,t)\}$, in which head/tail entities $h,t\in \mathcal{E}$ and relation $r \in \mathcal{R}$. 
To predict missing links, KGE maps entities and relations to distributed representations, and defines a score function to measure the plausibility of each triple.

Following a series of previous works~\cite{TransE, sun2019rotate, gao2020rotate3d, Rot-Pro}, we define the score function as a distance function $d_r(h,t)$. The distance of the positive triple $(h,r,t) \in \mathcal{D}$ is expected to be smaller than the corrupted negative triples $(h',r,t)$ or $(h,r,t')$, which can be generated by randomly replacing the entity $h$ or $t$ with other entities. 

In the training process, the self-adversarial negative sampling~\cite{sun2019rotate} is used to optimize the models in a contrastive way. 
Given a positive triple and its negative samples, the loss function is defined as follows: 
\begin{equation}
\begin{split}
L = &- \log{\sigma(\gamma-d_r(h,t))}\\
&- \sum_{i=1}^l p(h'_i,r,t'_i)\log{\sigma(d_r(h'_i,t'_i)-\gamma)}\\
&+ \frac{\lambda}{\lvert \mathcal{E} \rvert} \sum_{e \in \mathcal{E}} \Vert e \Vert_2^2,
\end{split}
\end{equation}
where $\gamma$ is a pre-defined margin, $\sigma$ is the sigmoid function, $l$ denotes the number of negative samples,  $(h'_i,r,t'_i)$ is a negative sample against $(h,r,t)$, $p(h'_i,r,t'_i)$ is the weight of negative sampling defined in \cite{sun2019rotate}, $\lambda$ is the regularization coefficient. For the sake of clarification, notations used in this paper are listed in Appendix \ref{notations}.

\begin{figure*}[t!]
\centering
\subfigure[]{\label{Fig-2a}\includegraphics[width=0.45\columnwidth]{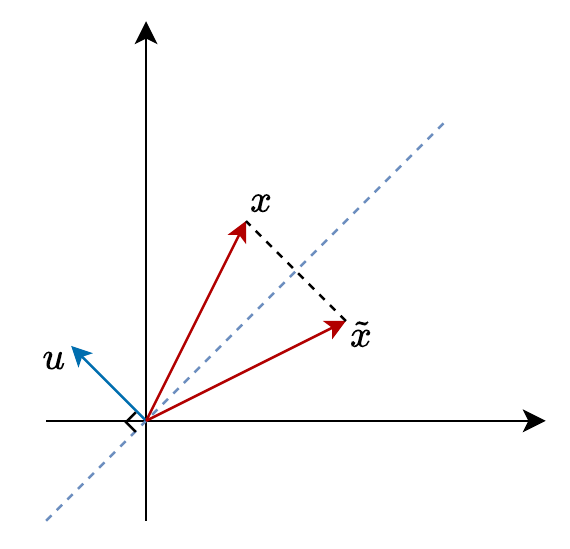}}
\quad
\subfigure[]{\label{Fig-2c}\includegraphics[width=0.45\columnwidth]{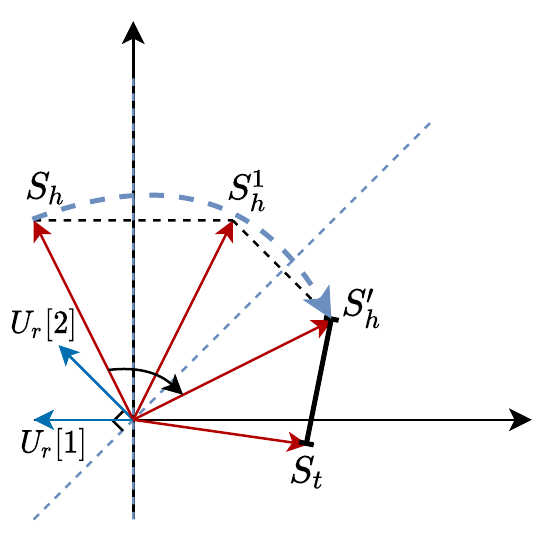}}
\quad
\subfigure[]{\label{Fig-2b}\includegraphics[width=0.45\columnwidth]{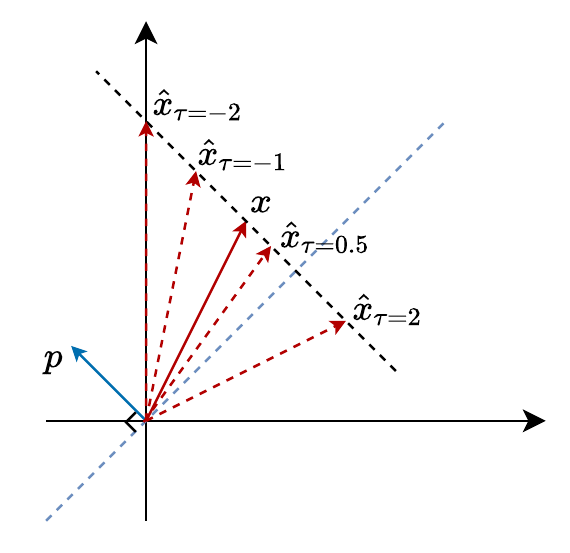}}
\quad
\subfigure[]{\label{Fig-2d}\includegraphics[width=0.45\columnwidth]{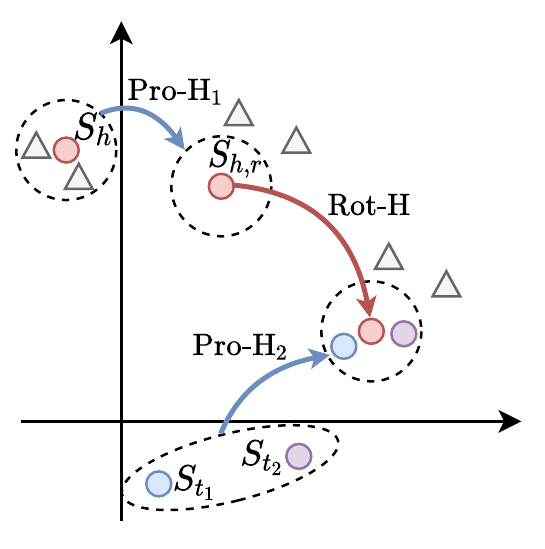}}
\vspace{-4mm}
\caption{(a) Householder reflection in 2-dimensional space; (b) HousE-r models relation $r$ as a 2-dimensional Householder rotation composed of two Householder reflections; (c) Modified Householder reflection in 2-dimensional space with different values of $\tau$; (d) To model $(h,r,t_1)$ and $(h,r,t_2)$, HousE first utilizes relational Householder projections ${\rm Pro\mbox{-}H}_1$ and ${\rm Pro\mbox{-}H}_2$ (blue lines) to change the relative distance between entities, such as increasing the distance between $S_h$ and the negative samples (Marked by triangles) and decreasing the distance between two positive tail entities $S_{t_1}$ and $S_{t_2}$, then HousE performs a relational Householder rotation ${\rm Rot\mbox{-}H}$ (red lines) from projected head embedding $S_{h,r}$ to the projected tail embedding. Note that we omit the row (element) indices $[i]$ for simplicity.}
\label{Fig-2}
\vskip -0.14in
\end{figure*}

\section{Methodology}
\subsection{HousE-r: Relational Householder Rotations}
\label{subsec: rotation} 
In the first step, we seek to develop a general framework to model relations as rotations in the space of any dimension $k$, going beyond \cite{sun2019rotate,gao2020rotate3d,zhang2019quaternion}, for better modeling capacity. 
In order to parameterize a $k$-dimensional rotation matrix, a straight-forward strategy is to randomly initialize a matrix and restrict it to a rotation matrix after every gradient descent update. 
However, such a naive approach may lead to the complicated optimization process and cannot fully cover the set of all $k \times k$ rotation matrices. 
In this paper, we theoretically prove that any $k$-dimensional rotations can be represented as $2\lfloor \frac{k}{2} \rfloor$ Householder reflections \cite{householder1958unitary}. 
Inspired by this theorem, we design an elegant parameterization based on Householder reflections to model $k$-dimensional rotations without any special optimizing procedure.












As the basic mathematical operator in this work, 
\normalem{\emph{Householder matrix}~\cite{householder1958unitary}} represents the reflection (\normalem{\emph{Householder reflection}}) about a hyperplane containing the origin.
Given a unit vector $u\in\mathbb{R}^{k}$, the $k \times k$ Householder matrix $H$, taking  $u$ as variable, is defined as $H(u)$:
\begin{equation}
H(u) = I - 2uu^\top, \label{House}
\end{equation}
where $\Vert u \Vert_2^2 = 1$ and $I$ is the $k\times k$ identity matrix. 
Geometrically, as shown in Figure \ref{Fig-2a}, the Householder matrix transforms $x$ to $\widetilde{x}$ by a reflection about the hyperplane orthogonal to the normal vector $u$:  
\begin{equation}
\begin{split}
\label{ref vector operation}
\widetilde{x} =H(u)x =x-2\braket{x,u}u, 
\end{split}
\end{equation}
where $\braket{\cdot}$ denotes the dot product.

Based on the Householder matrices, we can design a mapping to represent rotations.
Specifically, given a series of unit vectors $U =  \{u_c\}_{c=1}^{2n}$ where $u_c\in \mathbb{R}^k$ and $n$ is a positive integer, we define the mapping as follows:
\begin{equation}
\begin{split}
\label{map-Hrot}
    {\rm Rot\mbox{-}H}(U) = \prod_{c=1}^{2n}H(u_c).
\end{split}
\end{equation}
The output of ${\rm Rot\mbox{-}H}$ is a $k\times k$ orthogonal matrix with determinant $1$, i.e., a rotation matrix ~\cite{artin2016geometric}, since each Householder matrix $H(u_c)$ is orthogonal and its determinant is $-1$.
Moreover, we also prove that any rotation can be expressed as the composition of Householder reflections.

Formally, we have the following theorem:  
\begin{theorem}
When $n=\lfloor \frac{k}{2} \rfloor$, the image of ${\rm Rot\mbox{-}H}$ is the set of all $k\times k$ rotation matrices, i.e., ${\rm Image( Rot\mbox{-}H)}=\boldsymbol{\rm SO}(k)$, $\boldsymbol{\rm SO}(k)$ is the $k$-dimensional  special orthogonal group. (See proof in Appendix \ref{proof})
\label{thm-1}
\end{theorem}
Theorem \ref{thm-1} provides us a natural way to represent relations as high-dimensional rotations for better modeling capacity. 
We call such rotations composed of $2\lfloor \frac{k}{2} \rfloor$ Householder reflections as the \normalem{\emph{Householder rotations}}. 

Given a triple $(h,r,t)$, we denote the embeddings of head entity $h$ and tail entity $t$ as $S_h\in \mathbb{R}^{d\times k}$ and $S_t\in \mathbb{R}^{d\times k}$, where $d$ is the embedding size of entities and $k$ is the dimension size of each row vector.
Recall that in RotatE~\cite{sun2019rotate}, Rotate3D~\cite{gao2020rotate3d} and QuatE~\cite{zhang2019quaternion}, each element (row) in the entity embeddings is represented as a 2-dimensional, 3-dimensional, 4-dimensional vector (i.e., $k=2, 3, 4$). 
More generally, HousE-r represents each row of entity embeddings as a $k$-dimensional vector, i.e., $S_h[i], S_t[i]\in \mathbb{R}^k, i\in \{1,\ldots,d\}$.
To model each relation as a row-wise $k$-dimensional rotation between head and tail entities, the embedding of relation $r$ is denoted as $U_r\in \mathbb{R}^{d\times 2n \times k}$, where $n=\lfloor \frac{k}{2} \rfloor$. Each row $U_r[i]\in \mathbb{R}^{2n\times k}$ is composed of $2n$ $k$-dimensional unit vectors, i.e., $U_r[i][j]\in \mathbb{R}^k$ and $\Vert U_r[i][j] \Vert_2^2 = 1, j\in \{1,\ldots,2n\}$. 

We propose to parameterize the relational Householder rotations by using the mapping ${\rm Rot\mbox{-}H}$ in Equation (\ref{map-Hrot}).
Formally, for each triple $(h,r,t)$, HousE-r applies $r$-specific Householder rotations to the $i$-th row of head embedding $h$:  
\begin{equation}
\begin{split}
\label{HousE-equa}
S'_h[i]&={\rm Rot\mbox{-}H}({U_r[i]}) S_h[i]\\
&= \prod_{j=1}^{2n}H(U_r[i][j]) S_h[i].
\end{split}
\end{equation}

Based on Theorem \ref{thm-1}, any $k$-dimensional relational rotations can be represented by Equation (\ref{HousE-equa}).   
As illustrated in Figure \ref{Fig-2c},  a $2$-dimensional rotation can be viewed as the composition of $2$ Householder reflections.

\textbf{Distance function of HousE-r.} 
The distance function measures the distance between the rotated head entity embedding $S'_{h}$ and the tail entity embedding $S_t$: 
\begin{equation}
d_r(h,t)= \sum_{i=1}^d{\Vert S'_{h}[i] - S_t[i] \Vert_2}. \label{dist-func}
\end{equation}

\textbf{Modeling capability of HousE-r.} 
Theoretically, HousE-r can model and infer symmetry, antisymmetry, inversion and composition patterns.  
The definitions of these relation patterns are listed in Appendix \ref{definition} for clarity.  
\begin{claim}
\label{claim-1}
HousE-r can model the symmetry/antisymmetry pattern. (See proof in Appendix \ref{proof-claim-1})
\end{claim}

\begin{claim}
\label{claim-2}
HousE-r can model the inversion pattern. (See proof in Appendix \ref{proof-claim-2})
\end{claim}

\begin{claim}
\label{claim-3}
HousE-r can model the composition pattern. (See proof in Appendix \ref{proof-claim-3})
\end{claim}

\textbf{Efficient computation.}
The time complexity of Equation (\ref{HousE-equa}) is $O(2nk^2)$, in which $2n$ matrix-vector multiplications incur high computational costs.
However, it is worth noting that these matrix multiplications can be replaced by the vector operations. 
Formally, based on Equation (\ref{ref vector operation}), the $j$-th matrix-vector multiplication can be expressed as:
\begin{equation}
\label{iter-eff}
    \begin{split}
        S^j_h[i]&=H(U_r[i][j])S^{j-1}_h[i]\\
        &=S^{j-1}_h[i]-2\braket{S^{j-1}_h[i],U_r[i][j]}U_r[i][j],
    \end{split}
\end{equation}
where $S^{0}_h[i]=S_h[i]$. 
Through such iterated vector operations, the time complexity can be reduced to $O(2nk)$. 

\textbf{Connections to RotatE, Rotate3D and QuatE.} 
As shown in Table \ref{ability}, the rotations of RotatE, Rotate3D and QuatE are modeled in 2-dimensional, 3-dimensional and 4-dimensional spaces, respectively.
Geometrically, they can be viewed as the special cases of HousE-r by setting the rotation dimension $k$ to 2, 3 and 4, respectively.
For example, as shown in Figure \ref{Fig-2c}, HousE-r in 2-dimensional space is equivalent to RotatE since any rotation in a plane can be represented by two conjunctive Householder reflections. 
Moreover, unlike previous models restricting rotations to a fixed and low-dimensional space, HousE-r can easily model high-dimensional rotations by enlarging the value of $k$. 


\textbf{Limitation of Householder rotations.} 
On the other side of the coin, HousE-r is not the panacea as the pure Householder rotations suffer from the challenge of \normalem{\emph{indistinguishable  representations}} in modeling RMPs. 
Considering the ideal case of no-error embedding, 
we have the following deductions: 
\begin{itemize}
\item For a 1-to-N relation $r$, when $(h, r, t_1)$ and $(h, r, t_2)$ hold, $S_{t_1}[i]=S_{t_2}[i]$.
\item For an N-to-1 relation $r'$, when $(h_1, r', t)$ and $(h_2, r', t)$ hold, $S_{h_1}[i]=S_{h_2}[i]$.
\end{itemize}
One can see that the embeddings of different entities tend to be identical when facing the complex RMPs, leading to the uninformative  representations. 
Thus, it is meaningful to tackle this challenge in our proposal.


\subsection{HousE: Improved HousE-r with Relational Householder Projections}
\label{sec:modify}



To handle the sophisticated RMPs, some projection operations have been proposed 
and shown their effectiveness 
\cite{wang2014knowledge, lin2015learning, ji2015knowledge}. 
The relational projections enable the KGE models to generate relation-specific representations for each entity 
\cite{wang2014knowledge}. 
However, existing projections are irreversible transformations, leading to the failure in modeling inversion and composition patterns \cite{sun2019rotate}. 
Differently, we propose the novel invertible projections named  \normalem{\emph{Householder projections}} by modifying the vanilla Householder matrices to tackle the limitation of HousE-r.

More concretely, given a unit vector $p\in\mathbb{R}^{k}$, i.e., $\Vert p \Vert_2^2 = 1$ and a real scalar $\tau$, the $k \times k$ modified Householder matrix $M(p,\tau)$ is defined as:
\begin{equation}
M(p,\tau) = I - \tau pp^\top.  \label{House}
\end{equation}
Note that the modified Householder matrix $M(p,\tau)$ has $k-1$ eigenvalues equal to 1 and one eigenvalue equal to $1-\tau$. 
Thus, $M(p,\tau)$ is invertible when $\tau \neq 1$. 
Geometrically, the modified Householder matrix
transforms $x$ to $\hat{x}$ by a projection along the axis $p$:
\begin{equation}
\begin{split}
\label{pro vector operation}
\hat{x}=M(p,\tau)=x-\tau\braket{x,p}p,
\end{split}
\end{equation}
where $\tau$ determines the position of $\hat{x}$ on the axis $p$. 
Figure \ref{Fig-2b} illustrates several projected results with different values of $\tau$ in two-dimensional space.

Moreover, based on the modified Householder matrices, given a series of real scalars $T = \{\tau_c\}_{c=1}^{m}$ and unit vectors $P = \{p_c\}_{c=1}^{m}$ where $m$ is a positive integer and $p_c\in \mathbb{R}^k$, we define the mapping:
\begin{equation}
\begin{split}
\label{map-Hpro}
    {\rm Pro\mbox{-}H}(P, T) = \prod_{c=1}^{m}M(p_c, \tau_c).
\end{split}
\end{equation}
The output of ${\rm Pro\mbox{-}H}(P,T)$
is an invertible matrix 
since the product of invertible matrices is also an invertible matrix. 
We name such projections composed of $m$ modified Householder reflections as \normalem{\emph{Householder projections}}.
Different from the rigidly distance-preserving Householder rotations, the Householder projections can reversibly change the relative distance between two points, and thus provide a suitable solution for modeling RMPs without sacrificing the capability of modeling relation patterns.



Specifically, we incorporate the relational Householder rotations and relational Householder projections under a unified framework named HousE to enjoy the merits from both sides.
The relational Householder projections enable relation-specific representations for each entity and the relational Householder rotations enable high-dimensional rotations between projected entities.
As shown in Figure \ref{Fig-2d}, given the input triple $(h, r, t)$, HousE first learns the relation($r$)-specific representations $S_{h,r}$ and $S_{t,r}$ for head and tail entities via Householder projections, respectively. 
Then, $S_{h,r}$ is transformed by the high-dimensional Householder rotations to be close to $S_{t,r}$.  

In the phase of relational Householder projections, we define two types of parameters for each relation $r$: the axes $P_r\in \mathbb{R}^{d \times m \times k}$ and the scalars $T_r\in \mathbb{R}^{d \times m}$, where $m$ is a positive integer. Each row $P_{r}[i]\in \mathbb{R}^{m \times k}$ is composed of $m$ $k$-dimensional unit vectors (projection axes), i.e., $P_r[i][j]\in \mathbb{R}^k$ and $\Vert P_r[i][j] \Vert_2^2 = 1$. Each row $T_r[i]$ is composed of $m$ real values (projection scalars).

We propose to parameterize the relational Householder projections by using the mapping ${\rm Pro\mbox{-}H}$ in Equation (\ref{map-Hpro}).
Considering the head and tail parts of a relation usually
have different implicit types~\cite{SE}, HousE utilizes two sets of independent projection parameters $\{P_{r,1}, T_{r,1}\}$ and $\{P_{r,2}, T_{r,2}\}$ for each relation $r$ to project $h$ and $t$, respectively.  
Formally, For each triple $(h, r, t)$, HousE transforms each row of head entity $h$ and tail entity $t$ with $r$-specific Householder projections: 
\begin{equation}
\begin{split}
\label{Pro-equa}
S_{h,r}[i]&={\rm Pro\mbox{-}H}(P_{r,1}[i], T_{r,1}[i]) S_{h}[i]\\
&=\prod_{j=1}^{m}M(P_{r,1}[i][j], T_{r,1}[i][j]) S_{h}[i],\\
S_{t,r}[i]&= {\rm Pro\mbox{-}H}(P_{r,2}[i], T_{r,2}[i]) S_{t}[i]\\
&=\prod_{j=1}^{m}M(P_{r,2}[i][j], T_{r,2}[i][j]) S_{t}[i]. 
\end{split}
\end{equation}

\renewcommand{\algorithmiccomment}[1]{/* #1 */}
\begin{algorithm}[tb]
   \caption{Forward procedure of HousE}
   \label{alg:algo}
\begin{algorithmic}[1]
   \STATE {\bfseries Input:} An input triple $(h,r,t)$, head (tail) entity embedding $S_{h}$ ($S_{t}$), projection parameters $\{T_{r,1}, T_{r,2}\}$ and $\{P_{r,1}, P_{r,2}\}$, rotation parameters $U_{r}$, embedding size $d$, rotation dimension $k$, number of modified Householder reflections $m$. 
   \STATE {\bfseries Output:} Distance $\delta$
   \STATE $\delta \gets 0$
   \STATE $n \gets \lfloor \frac{k}{2} \rfloor$
   \FOR{$i=1$ {\bfseries to} $d$}
   \STATE \begin{varwidth}[t]{\linewidth}
   \COMMENT{Relational Householder projections}\par
      $S_{h,r}[i] \gets \prod_{j=1}^{m}M(P_{r,1}[i][j], T_{r,1}[i][j]) S_{h}[i]$\par
      $S_{t,r}[i] \gets \prod_{j=1}^{m}M(P_{r,2}[i][j], T_{r,2}[i][j]) S_{t}[i]$
      \end{varwidth}
   \STATE \begin{varwidth}[t]{\linewidth}
   \COMMENT{Relational Householder rotations}\par
   $S'_{h,r}[i] \gets \prod_{j=1}^{2n}H(U_r[i][j]) S_{h,r}[i]$
   \end{varwidth}
   \STATE $\delta \gets \delta + \Vert S'_{h,r}[i]-S_{t,r}[i]\Vert_2$
   \ENDFOR
   \STATE {\bfseries Return:} $\delta$
\end{algorithmic}
\end{algorithm}
\begin{table*}[t]
\caption{Link prediction results on WN18 and FB15k. Best results are in \textbf{bold} and second best results are \underline{underlined}. $[\dagger]$: Results are taken from~\cite{ConvKB}; $[\diamond]$: Results are taken from~\cite{kadlec2017knowledge}. Other results are taken from the original papers.}
\label{WN18-and-FB15k}
\begin{center}
\begin{small}
\resizebox{0.65\textwidth}{!}{
\begin{tabular}{lcccccccccc}
\toprule
 & \multicolumn{5}{c}{WN18}                                                                                                                                                             & \multicolumn{5}{c}{FB15k}                                                                                                                                       \\ \cmidrule(r){2-6} \cmidrule(r){7-11}
 Model                      & MR                               & MRR                                & H@1                                & H@3                                & H@10                               & MR                              & MRR                                & H@1                                & H@3                                & H@10                                           \\
\midrule
TransE$\dagger$                 & -                                & .495                              & .113                              & .888                              & .943                              & -                               & .463                              & .297                              & .578                              & .749                                  \\
DistMult$\diamond$               & 655                              & .797                              & -                                 & -                                 & .946                              & 42.2                            & .798                              & -                                 & -                                 & .893                            \\
ComplEx                & -                                & .941                              & .936                              & .945                              & .947                              & -                               & .692                              & .599                              & .759                              & .84                    \\
ConvE                  & 374                              & .943                              & .935                              & .946                              & .956                              & 51                              & .657                              & .558                              & .723                              & .831                \\
RotatE                 & 309                              & .949                              & .944                              & .952                              & .959                              & 40                              & .797                              & .746                              & .830                              & .884                       \\
Rotate3D                 & 214                              & .951                              & .945                              & .953                              & .961                              & \underline{39}                              & .789                              & .728                              & .832                              & .887                       \\
QuatE                  & 388                              & .949                              & .941                              & .954                              & .960                              & 41                              & .770                              & .700                              & .821                              & .878                         \\
DualE                  & -                                & .951                              & .945                              & .956                              & .961                              & -                               & .790                              & .734                              & .829                              & .881                         \\
\midrule
HousE-r & \underline{155} & \underline{.953} & \underline{.947} & \underline{.956} & \underline{.964} & \underline{39} & \underline{.807} & \underline{.758} & \underline{.839} & \underline{.893} \\
HousE                   &\textbf{137} & \textbf{.954} & \textbf{.948} & \textbf{.957} & \textbf{.964} & \textbf{38} & \textbf{.811} & \textbf{.759} &  \textbf{.847} & \textbf{.898}\\
\bottomrule
\end{tabular}}
\end{small}
\end{center}
\vspace{-5mm}
\end{table*}

\begin{table*}[t]
\caption{Link prediction results on WN18RR, FB15k-237 and YAGO3-10. Best results are in \textbf{bold} and second best results are \underline{underlined}. $[\dagger]$: Results are taken from~\cite{ConvKB}; $[\diamond]$: Results are taken from~\cite{dettmers2018convolutional}. Other results are taken from the corresponding original papers.}
\label{WN18RR-and-FB15k237}
\begin{center}
\begin{small}
\resizebox{0.95\textwidth}{!}{
\begin{tabular}{lccccccccccccccc}
\toprule
                        & \multicolumn{5}{c}{WN18RR}                                                                                                                                                                           & \multicolumn{5}{c}{FB15k-237}               & \multicolumn{5}{c}{YAGO3-10}  \\ \cmidrule(r){2-6} \cmidrule(r){7-11} \cmidrule(r){12-16}
Model & MR                                   & MRR                                   & H@1                                   & H@3                                   & H@10                                  & MR                                  & MRR                                   & H@1                                   & H@3                                   & H@10                                  & MR                   & MRR                  & H@1                  & H@3                  & H@10                 \\
\midrule
TransE$\dagger$                  & 3384                                 & .226                                 & -                                     & -                                     & .501                                 & 357                                 & .294                                 & -                                     & -                                     & .465                                  & -                    & -                    & -                    & -                    & -                    \\
DistMult$\diamond$                & 5110                                 & .43                                  & .39                                  & .44                                  & .49                                  & 254                                 & .241                                 & .155                                 & .263                                 & .419                                 & 5926                 & .34                  & .24                  & .38                  & .54                  \\
ComplEx$\diamond$                 & 5261                                 & .44                                  & .41                                  & .46                                  & .51                                  & 339                                 & .247                                 & .158                                 & .275                                 & .428                                 & 6351                 & .36                  & .26                  & .4                   & .55                  \\
ConvE$\diamond$                   & 4187                                 & .43                                  & .40                                  & .44                                  & .52                                  & 224                                 & .325                                 & .237                                 & .356                                 & .501                                 & 1671                 & .44                  & .35                  & .49                  & .62                  \\
RotatE                  & 3340                                 & .476                                 & .428                                 & .492                                 & .571                                 & 177                                 & .338                                 & .241                                 & .375                                 & .533                                 & 1767                 & .495                 & .402                 & .55                  & .67                  \\
Rotate3D                  & 3328                                 & .489                                 & .442                                 & .505                                 & .579                                 & 165                                 & .347                                 & .250                                 & \underline{.385}                                 & \underline{.543}                                 & -                 & -                 & -                 & -                  & -                  \\
QuatE                   & 3472                                 & .481                                 & .436                                 & .500                                 & .564                                 & 176                                 & .311                                 & .221                                 & .342                                 & .495                                 & -                    & -                    & -                    & -                    & -                    \\
DualE                   & -                                    & .482                                 & .440                                 & .500                                 & .561                                 & -                                   & .330                                 & .237                                 & .363                                 & .518                                 & -                    & -                    & -                    & -                    & -                    \\
Rot-Pro                   & 2815                                    & .457                                 & .397                                 & .482                                 & .577                                 & 201                                   & .344                                 & .246                                 & .383                                 & .540                         & 1797                    & .542                    & .443                    & .596                    & .669                    \\
\midrule
HousE-r & \underline{1885} & \underline{.496} & \underline{.452} & \underline{.511} & \underline{.585} & \underline{165} & \underline{.348} & \underline{.254} & .384 & .534 & \underline{1449} & \underline{.565} & \underline{.487} & \underline{.616} & \underline{.703} \\
HousE                   &\textbf{1303} & \textbf{.511} & \textbf{.465} & \textbf{.528} & \textbf{.602} & \textbf{153} & \textbf{.361} & \textbf{.266} &  \textbf{.399} & \textbf{.551} & \textbf{1415} & \textbf{.571} & \textbf{.491} & \textbf{.620} & \textbf{.714}\\
\bottomrule
\end{tabular}}
\end{small}
\end{center}
\vspace{-5mm}
\end{table*}

After that, HousE models the row-wise Householder rotations between the projected head point $S_{h,r}$ and projected tail point $S_{t,r}$, which is the same as the one in Equation (\ref{HousE-equa}). 
If $(h,r,t)$ holds, we expect the rotated head point $S'_{h,r}[i]={\rm Rot\mbox{-}H}({U_r[i]})  S_{h,r}[i] \approx S_{t,r}[i]$, where $U_r[i]$ is composed of $2\lfloor \frac{k}{2} \rfloor$ $r$-specific Householder reflections.


As shown in Algorithm \ref{alg:algo}, for each triple $(h,r,t)$, HousE first utilizes the relational Householder projection to generate $r$-specific representations $S_{h,r}$ and $S_{t,r}$ for $h$ and $t$, as in line $6$. Then, HousE applies the relational Householder rotation to the projected head embedding $S_{h,r}$, as in line $7$. The rotated result $S'_{h,r}$ is expected to be close to the projected tail embedding $S_{t,r}$. Note that we replace the matrix-vector multiplications in line $6$ and $7$ with the vector operations in Equation (\ref{pro vector operation}) and (\ref{ref vector operation}) for efficient computation.

The learnable parameters of HousE include $\{S_e\}_{e\in \mathcal{E}}$ and $\{U_r, P_{r,1}, P_{r,2}, T_{r,1}, T_{r,2}\}_{r\in \mathcal{R}}$. Compared to previous models \cite{sun2019rotate,zhang2019quaternion}, the extra cost is proportional to the number of relation types, which is usually much smaller than the number of entities. Therefore, the total number of parameters in HousE is about $O(dk\left | \mathcal{E} \right |)$.

\textbf{Distance function of HousE.} 
For each triple $(h,r,t)$, the distance function of HousE is defined as:
\begin{equation}
d_r(h,t)= \sum_{i=1}^d{\Vert S'_{h,r}[i] - S_{t,r}[i] \Vert_2}. \label{dist-func}
\end{equation}

\textbf{Modeling capability of HousE.} HousE can model and infer all the relation patterns and RMPs as shown in Table \ref{ability} (we also discuss other relation patterns in Appendix \ref{discussion}). Formally, we can achieve the following claims:
\begin{claim}
\label{claim-4}
HousE can model the symmetry/antisymmetry pattern. (See proof in Appendix \ref{proof-claim-4})
\end{claim}
\begin{claim}
\label{claim-5}
HousE can model the inversion pattern. (See proof in Appendix \ref{proof-claim-5})
\end{claim}
\begin{claim}
\label{claim-6}
HousE can model the composition pattern. (See proof in Appendix \ref{proof-claim-6})
\end{claim}
\begin{claim}
\label{claim-8}
HousE can model the relation mapping properties. (See proof in Appendix \ref{proof-claim-8})
\end{claim}

\textbf{Connections to TransH, TransR and TransD.} Previous works such as TransH, TransR and TransD also focus on designing the projection operations to ensure that the same entity has different representations under different relations. 
However, as shown in Table \ref{ability}, these methods will undermine the ability to infer inversion and composition patterns due to the irreversible projection operations. Note that, the projection operation of TransH is a special case of HousE if we set the scalar $\tau = 1$, which essentially is the irreversible transformation. 
Different from these works, HousE utilizes an invertible matrix derived by a series of modified Householder matrices to generate relation-specific entity representations. Such invertible projections enable our proposal to model relation mapping properties without sacrificing the capability in modeling relation patterns. 

\section{Experiment}

\subsection{Experimental Setup}
\textbf{Datasets.} We evaluate our proposals on five widely-used benchmarks:  WN18~\cite{TransE}, FB15k~\cite{TransE}, WN18RR~\cite{dettmers2018convolutional}, FB15k-237~\cite{toutanova2015observed} and YAGO3-10~\cite{mahdisoltani2014yago3}. 
Refer to Appendix \ref{datasets app} for more details.

\textbf{Baselines.} We compare our models with a number of  baselines. For non-rotation models, we report TransE~\cite{TransE}, DistMult~\cite{yang2014embedding}, ComplEx~\cite{trouillon2016complex} and ConvE~\cite{dettmers2018convolutional}. For rotation-based models, we report RotatE~\cite{sun2019rotate}, Rotate3D~\cite{gao2020rotate3d}, QuatE~\cite{zhang2019quaternion}, DualE~\cite{cao2021dual} and Rot-Pro~\cite{Rot-Pro}. 

\textbf{Implementation details.} 
To ensure fair comparisons, we set a smaller embedding size $d$ for HousE-r and HousE, so that the total numbers of parameters are comparable to  baselines. More details can be found in Appendix \ref{Imp details}.



\subsection{Main Results}
The experimental results are summarized in Table \ref{WN18-and-FB15k} and Table \ref{WN18RR-and-FB15k237}. Compared to all the baselines, both HousE-r and HousE achieve SOTA performance, demonstrating the effectiveness of the Householder framework.

Table \ref{WN18-and-FB15k} shows the results on WN18 and FB15k, from which we observe that even with only Householder rotations, HousE-r already consistently outperforms the baselines over both datasets. 
Moreover, by combining Householder rotations and Householder projections together, HousE further  
achieves new state-of-the-art results on both WN18 and FB15k datasets. 
Considering that the main relation patterns in WN18 and FB15k are symmetry, antisymmetry and inversion, the superior performance of HousE-r and HousE reveals their effectiveness in modeling these patterns.


Table \ref{WN18RR-and-FB15k237} summarizes the results on WN18RR, FB15k-237 and YAGO3-10.
On these datasets, HousE-r surpasses most of the baselines. The only comparable exception is Rotate3D on FB15k-237 which models relations as 3-d rotations. 
However, HousE-r uses much less parameters than Rotate3D as shown in Appendix \ref{Imp details} and achieves similar performance, which also verifies the superior modeling capacity of Householder rotations.
The improvements over existing rotations-based baselines (RotatE, Rotate3D, QuatE and DualE) demonstrate the superiority of  high-dimensional rotations. 
Moreover, HousE consistently outperforms HousE-r along with all the baselines by a large margin on the three datasets across all metrics, benefiting from the ability to model relational mapping properties.


\begin{table}[t]
\caption{MRR for the models tested on each relation of WN18RR.}
\label{Case-WN18RR}
\vspace{-2.5mm}
\begin{center}
\begin{small}
\resizebox{\columnwidth}{!}{
\begin{tabular}{lccccc}
\toprule
Relation Name   & RotatE  & QuatE  & HousE-r  & HousE  \\
\midrule
hypernym         & 0.154 & 0.172 & {\ul 0.182}    & \textbf{0.207} \\
instance\_hypernym   & 0.324 & 0.362 & {\ul 0.395}    & \textbf{0.440} \\
member\_meronym      & 0.255 & 0.236 & {\ul 0.275}    & \textbf{0.312} \\
synset\_domain\_topic\_of    & 0.334 & 0.395 & {\ul 0.396}    & \textbf{0.428} \\
has\_part            & 0.205 & 0.210 & {\ul 0.217}    & \textbf{0.232} \\
member\_of\_domain\_usage     & 0.277 & 0.372 & {\ul 0.415}    & \textbf{0.453} \\
member\_of\_domain\_region    & 0.243 & 0.140 & {\ul 0.281}    & \textbf{0.395} \\
derivationally\_related\_form  & 0.957 & 0.952 & {\ul 0.958} & \textbf{0.958}    \\
also\_see           & 0.627 & 0.607 & {\ul 0.638} & \textbf{0.640}    \\
verb\_group                  & 0.968 & 0.930 & {\ul 0.968}    & \textbf{0.968} \\
similar\_to             & 1.000 & 1.000 & {\ul 1.000}    & \textbf{1.000} \\
\bottomrule
\end{tabular}}
\end{small}
\end{center}
\vspace{-4mm}
\end{table}
\subsection{Fine-grained Performance Analysis}
To further verify the modeling capacity of our proposal from a fine-grained perspective, we report the performance on each relation of WN18RR following  \cite{zhang2019quaternion}. As shown in Table \ref{Case-WN18RR}, 
compared to two rotation-based baselines RotatE and QuatE, we observe that:

(1) HousE-r surpasses all the baselines on all 11 relation types, confirming the superior modeling capacity of the Householder rotations.

(2) By incorporating the Householder projections, HousE achieves more significant improvements on the challenging 1-to-N and N-to-1 relations. For example, HousE outperforms RotatE on 1-to-N relation  \normalem{\emph{member\_of\_domain\_region}}  and N-to-1 relation  \normalem{\emph{instance\_hypernym}} with 62.55\% and 35.80\% relative gains, respectively.

\subsection{Capability of Modeling RMPs}
In order to further demonstrate the effectiveness of HousE in modeling RMPs, we report the detailed results of our proposal on different RMPs\footnote[2]{Following \cite{sun2019rotate}, for each relation $r$, we compute the average number of heads per tail ($hpt_r$) and the average number of tails per head $(tph_r)$. If $hpt_r\textless 1.5$ and $tph_r\textless 1.5$, $r$ is treated as 1-to-1; if $hpt_r\ge 1.5$ and $tph_r\ge 1.5$, $r$ is treated as N-to-N; if $hpt_r\textless 1.5$ and $tph_r\ge 1.5$, $r$ is treated as 1-to-N; if $hpt_r\ge 1.5$ and $tph_r\textless 1.5$, $r$ is treated as N-to-1.} in FB15k-237.

Table \ref{RMPs} exhibits the results on different types of RMPs. 
One can see that HousE outperforms RotatE across all RMP types.
For example, on the challenging N-to-1 (predicting head) and 1-to-N (predicting tail) tasks, HousE achieves 29.55\% and 21.13\% relative improvements over RotatE.
Such advanced performance of HousE owes to the powerful modeling capability of the Householder projections. 

\begin{table}[t]
\small
\caption{MRR for the models tested on RMPs in FB15k-237.}
\label{RMPs}
\vskip 0.14in
\begin{center}
\resizebox{0.7\columnwidth}{!}{
\begin{tabular}{c|l|cc}
\toprule
Task                             & RMPs & \multicolumn{1}{l}{RotatE}              & HousE \\
\midrule
\multirow{4}{*}{\makecell[c]{Predicting\\Head\\(MRR)}} & 1-to-1                        & 0.498                                 & \textbf{0.514} \\
                                 & 1-to-N                        & 0.475                                       &   \textbf{0.479} \\
                                 & N-to-1                       & 0.088                                      & \textbf{0.114} \\
                                 & N-to-N                       & 0.260                                        & \textbf{0.286} \\
\midrule
\multirow{4}{*}{\makecell[c]{Predicting\\Tail\\(MRR)}} & 1-to-1                         & 0.490                       & \textbf{0.502} \\
                                 & 1-to-N                        & 0.071                          & \textbf{0.086} \\
                                 & N-to-1                        & 0.747                       & \textbf{0.778} \\
                                 & N-to-N                        & 0.367                       & \textbf{0.392} \\
\bottomrule
\end{tabular}}
\end{center}
\vspace{-4mm}
\end{table}






\subsection{Hyperparameter Sensitivity Analysis}

\textbf{Dimension of rotations.}
To verify the expressiveness of high-dimensional rotations, we conduct experiments for our models under varying rotation dimension $k$. Figure \ref{Fig-3a} and \ref{Fig-3b} show the results on WN18RR and FB15k-237. 

As expected, on both datasets, HousE-r and HousE rotated in higher-dimensional spaces achieve better performance than the ones rotated in lower-dimensional spaces, since the high-dimensional rotations bring the superior modeling capacity.
Moreover, HousE consistently surpasses HousE-r by a large margin across all rotation dimensions, demonstrating the effectiveness of the integrated Householder projections. For example, on WN18RR, HousE with 4-dimensional rotations already outperforms HousE-r with 12-dimensional rotations.


\textbf{Number of modified Householder matrices.}
As shown in Equation (\ref{map-Hpro}), a Householder projection is composed of $m$ modified Householder matrices.
Here we investigate the impact of $m$ on the performance (MRR) of HousE.
Figure \ref{Fig-3c} and \ref{Fig-3d} show the results on WN18RR and FB15k-237. 

With the increase of $m$, the performance of HousE first improves and then drops on both datasets. 
This is because the larger $m$ provides greater projection capability, but the overcomplicated projections also    aggravate the risk of overfitting.  
Moreover, the values of $m$ for the best performance on the two datasets are different ($m=1$ on WN18RR and $m=6$ on FB15k-237) due to the distinct graph densities.
Specifically, WN18RR is a sparse KG dataset with the average degree of $2.19$, while FB15k-237 is a much denser KG with the average degree of $18.71$. Thus, the larger $m$ is needed for modeling the richer graph information in FB15k-237.

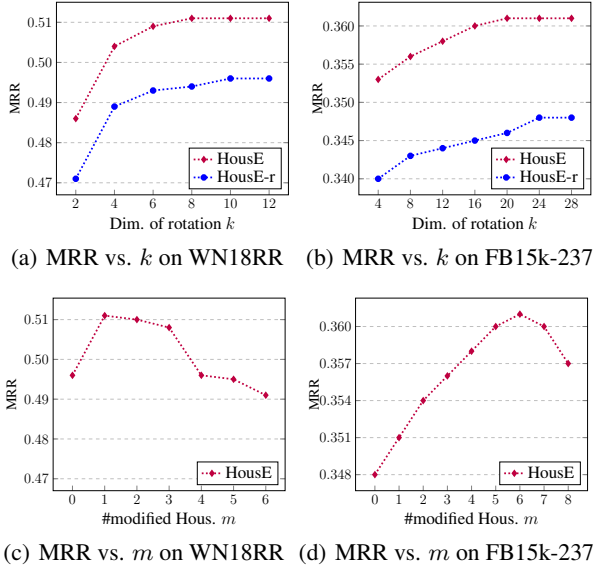
\begin{figure}[t!]
\centering
\subfigure[MRR vs. $k$ on WN18RR]{\label{Fig-3a}
            \begin{tikzpicture}[font=\Large, scale=0.45]
                \begin{axis}[
                    legend cell align={left},
                    legend style={nodes={scale=1.0, transform shape}},
                    xlabel={Dim. of rotation $k$},
                    xtick pos=left,
                    tick label style={font=\large},
                    ylabel style={font=\large},
                    ylabel={MRR},
                    xtick={2, 4, 6, 8, 10, 12},
                    xticklabels={$2$, $4$, $6$, $8$, $10$, $12$},
                    ytick={0.46, 0.47,0.48,0.49,0.50,0.51, 0.52},
                    yticklabels={$0.46$, $0.47$,$0.48$,$0.49$,$0.50$,$0.51$, $0.52$},
                    legend pos=south east,
                    ymajorgrids=true,
                    grid style=dashed
                ]
                \addplot[
                    color=purple,
                    dotted,
                    mark options={solid},
                    mark=diamond*,
                    line width=1.5pt,
                    mark size=2pt
                    ]
                    coordinates {
                    (2, 0.486)
                    (4, 0.504)
                    (6, 0.509)
                    (8, 0.511)
                    (10, 0.511)
                    (12, 0.511)
                    };
                    \addlegendentry{HousE}
                \addplot[
                    color=blue,
                    dotted,
                    mark options={solid},
                    mark=*,
                    line width=1.5pt,
                    mark size=2pt
                    ]
                    coordinates {
                    (2, 0.471)
                    (4, 0.489)
                    (6, 0.493)
                    (8, 0.494)
                    (10, 0.496)
                    (12, 0.496)
                    };
                    \addlegendentry{HousE-r}
                \end{axis}
                \end{tikzpicture}
    }
    \subfigure[MRR vs. $k$ on FB15k-237]{\label{Fig-3b}
            \begin{tikzpicture}[font=\Large,scale=0.45]
                \begin{axis}[
                    legend cell align={left},
                    legend style={nodes={scale=1.0, transform shape}},
                    xlabel={Dim. of rotation $k$},
                    xtick pos=left,
                    tick label style={font=\large},
                    ylabel style={font=\large},
                    ylabel={MRR},
                    xtick={4, 8, 12, 16, 20, 24, 28},
                    xticklabels={$4$, $8$, $12$, $16$, $20$, $24$,$28$},
                    ytick={0.340, 0.345,0.350,0.355,0.360},
                    yticklabels={$0.340$, $0.345$, $0.350$,$0.355$,$0.360$},
                    legend pos=south east,
                    ymajorgrids=true,
                    grid style=dashed
                ]
                \addplot[
                    color=purple,
                    dotted,
                    mark options={solid},
                    mark=diamond*,
                    line width=1.5pt,
                    mark size=2pt
                    ]
                    coordinates {
                    (4, 0.353)
                    (8, 0.356)
                    (12, 0.358)
                    (16, 0.360)
                    (20, 0.361)
                    (24, 0.361)
                    (28, 0.361)
                    };
                    \addlegendentry{HousE}
                \addplot[
                    color=blue,
                    dotted,
                    mark options={solid},
                    mark=*,
                    line width=1.5pt,
                    mark size=2pt
                    ]
                    coordinates {
                    (4, 0.340)
                    (8, 0.343)
                    (12, 0.344)
                    (16, 0.345)
                    (20, 0.346)
                    (24, 0.348)
                    (28, 0.348)
                    };
                    \addlegendentry{HousE-r}
                \end{axis}
                \end{tikzpicture}
    }
    
    \subfigure[MRR vs. $m$ on WN18RR]{\label{Fig-3c}
            \begin{tikzpicture}[font=\Large,scale=0.45]
                \begin{axis}[
                    legend cell align={left},
                    legend style={nodes={scale=1.0, transform shape}},
                    xlabel={\#modified Hous. $m$},
                    xtick pos=left,
                    tick label style={font=\large},
                    ylabel style={font=\large},
                    ylabel={MRR},
                    ymin=0.467,
                    xtick={0, 1, 2, 3, 4, 5, 6},
                    xticklabels={$0$, $1$, $2$, $3$, $4$, $5$, $6$},
                    ytick={0.46, 0.47,0.48,0.49,0.50,0.51, 0.52},
                    yticklabels={$0.46$, $0.47$,$0.48$,$0.49$,$0.50$,$0.51$, $0.52$},
                    legend pos=south east,
                    ymajorgrids=true,
                    grid style=dashed
                ]
                \addplot[
                    color=purple,
                    dotted,
                    mark options={solid},
                    mark=diamond*,
                    line width=1.5pt,
                    mark size=2pt
                    ]
                    coordinates {
                    (0, 0.496)
                    (1, 0.511)
                    (2, 0.510)
                    (3, 0.508)
                    (4, 0.496)
                    (5, 0.495)
                    (6, 0.491)
                    };
                    \addlegendentry{HousE}
                \end{axis}
                \end{tikzpicture}
    }
    \subfigure[MRR vs. $m$ on FB15k-237]{\label{Fig-3d}
            \begin{tikzpicture}[font=\Large,scale=0.45]
                \begin{axis}[
                    legend cell align={left},
                    legend style={nodes={scale=1.0, transform shape}},
                    xlabel={\#modified Hous. $m$},
                    xtick pos=left,
                    tick label style={font=\large},
                    ylabel style={font=\large},
                    ylabel={MRR},
                    xtick={0, 1, 2, 3, 4, 5, 6, 7, 8},
                    xticklabels={$0$, $1$, $2$, $3$, $4$, $5$, $6$, $7$,$8$},
                    ytick={0.348, 0.351, 0.354, 0.357, 0.360},
                    yticklabels={$0.348$, $0.351$, $0.354$, $0.357$,$0.360$},
                    legend pos=south east,
                    ymajorgrids=true,
                    grid style=dashed
                ]
                \addplot[
                    color=purple,
                    dotted,
                    mark options={solid},
                    mark=diamond*,
                    line width=1.5pt,
                    mark size=2pt
                    ]
                    coordinates {
                    (0, 0.348)
                    (1, 0.351)
                    (2, 0.354)
                    (3, 0.356)
                    (4, 0.358)
                    (5, 0.360)
                    (6, 0.361)
                    (7, 0.360)
                    (8, 0.357)
                    };
                    \addlegendentry{HousE}
                \end{axis}
                \end{tikzpicture}
    }
    \vspace{-1mm}
    \caption{(a) and (b) show the MRR results of HousE and HousE-r with varying rotation dimensions on WN18RR and FB15k-237; (c) and (d) show the MRR results of HousE with different numbers of modified Householder matrices on WN18RR and FB15k-237.}
    \vspace{-2mm}
\end{figure}

\subsection{Superiority of Householder Projections}
To verify the effectiveness of the proposed Householder projections, we design two variants of HousE by replacing the Householder projections with previous irreversible projections used in TransH~\cite{wang2014knowledge} and TransR~\cite{lin2015learning}, dubbed HousH and HousR respectively.
Table \ref{Variants} shows the experimental results on WN18RR and FB15k-237. Compared to HousE-r without any projections, the performance of HousH and HousR is barely improved on FB15k-237, and even degraded on WN18RR.
It reveals that the irreversible projections may hinder the modeling capability.
Moreover, HousE significantly outperforms HousH and HousR on both datasets, demonstrating the superiority of the invertible Householder projections in HousE.

\subsection{Additional Translations}
\label{translation plus} 
To explore the potential of our proposal, we also incorporate translations \cite{TransE} into HousE-r and HousE, dubbed HousE-r$^+$ and HousE$^+$ respectively. 
The translations are directly deployed after the Householder rotations.  
From Table \ref{Variants}, we see that these two variants both outperform their original versions. 
This is because the translations provide a natural way to represent the hierarchical property of KGs \cite{TransE}, which also endows our proposal with more comprehensive modeling capacity. 

\begin{table}[t]
\caption{Performance of different variants.}
\label{Variants}
\vskip 0.1in
\begin{center}
\begin{small}
\resizebox{0.75\columnwidth}{!}{
\begin{tabular}{lcccc}
\toprule
         & \multicolumn{2}{c}{WN18RR} & \multicolumn{2}{c}{FB15k-237} \\ \cmidrule(r){2-3} \cmidrule(r){4-5}
Variants & MRR          & H@10        & MRR           & H@10          \\ 
\midrule
HousH    & .491         & .584        & .347              &     .537          \\
HousR    & .488         & .580        & .349              &     .538          \\
\midrule
HousE-r  & .496         & .585        & .348              & .534               \\
HousE-r$^+$ & .500         & .591        & .351              & .538              \\
\midrule
HousE    & .511         & .602        & .361              & .551              \\
HousE$^+$   & .514         & .606        & .366              & .552              \\
\bottomrule
\end{tabular}}
\end{small}
\end{center}
\vspace{-5mm}
\end{table}


\section{Related Work}


\textbf{Translation-based models.} TransE \cite{TransE} is the first model that represents each relation as a translation between entities.
This simple model is effective in modeling antisymmetry, inversion and composition patterns, but fails in handling symmetry pattern and RMPs. 
To tackle TransE's limitations, a set of variants \cite{wang2014knowledge, lin2015learning, ji2015knowledge, xiao2015transa} are subsequently proposed. 
TransH \cite{wang2014knowledge} projects entities to a relation-specific hyperplane and performs translation on this hyperplane. TransR \cite{lin2015learning} models entities and relations in distinct spaces and conducts relation-specific projections with normal linear transformations. 
However, these models lose the ability to model inversion and composition patterns since irreversible linear transformations are performed on head and tail entities \cite{sun2019rotate}. 

\textbf{Rotation-based models.} Following ComplEx \cite{trouillon2016complex} which extends DistMult \cite{yang2014embedding} to complex number systems, RotatE \cite{sun2019rotate} represents each relation as a 2-dimensional rotation in complex plane to model symmetry, antisymmetry, inversion and composition patterns.
Rotate3D \cite{gao2020rotate3d} and QuatE \cite{zhang2019quaternion} extend the rotations to 3-dimensional and 4-dimensional spaces by introducing the quaternion number system. 
Recently, DualE \cite{cao2021dual} utilizes dual quaternions to combine translations and rotations in 3-d space for modeling multiple relations. 




\textbf{Neural-network-based models.} There are also some models using neural networks for KGE. R-GCN \cite{R-GCN} introduces graph neural networks as the graph encoders. ConvE \cite{dettmers2018convolutional} exploits convolution operations to facilitate the score calculation. 
However, such methods lack of explicit geometrical explanations on modeling relation patterns and RMPs. 

\section{Conclusion}

In this paper, we propose HousE, a novel KGE framework based on Householder parameterization.
HousE models relations as high-dimensional Householder rotations to capture crucial relation patterns. Moreover, with Householder projections, HousE generates relation-specific embeddings for each entity to model RMPs. Experimental results on five datasets   demonstrate the superiority of our proposal. 





\section*{Acknowledgements}
This work is supported in part by the National Key Research and Development Program of China (no. 2021ZD0112400), and also in part by the National Natural Science Foundation of China (no. U1811463).

\bibliography{example_paper}
\bibliographystyle{icml2022}

\newpage
\appendix
\onecolumn
\section{Notations}
\label{notations}
\begin{table}[h]
\caption{Notations used in this paper.}
\label{params}
\vspace{0.15in}
\begin{center}
\begin{small}
\renewcommand{\arraystretch}{1.5}
\begin{tabular}{ccl}
\toprule
Symbol     & Shape & Description  \\
\midrule
$\mathcal{E}$ & - & Set of entities \\
$\mathcal{R}$ & - &  Set of relations \\
$\mathcal{D}$ & - &  Set of factual triples \\
$h, t$  & -  & Head entity and tail entity \\
$r$     & -    & Relation type \\
$d$     &  $\mathbb{R}$      &  Embedding size\\
$k$     &  $\mathbb{R}$      &  Rotation dimension\\
$m$   &   $\mathbb{R}$      &  Number of modified Householder matrices\\
$S_e$  & $\mathbb{R}^{d \times k}$    & Representation of entity $e\in \mathcal{E}$ \\
$S_{e,r}$ & $\mathbb{R}^{d \times k}$ & $r$-specific representation of entity $e\in \mathcal{E}$ \\
$U_r$     &  $\mathbb{R}^{d \times 2\lfloor \frac{k}{2}\rfloor \times k}$    & Param. of $r\in \mathcal{R}$ for Householder rotation \\
$P_r$   &  $\mathbb{R}^{d \times m \times k}$   & Param. of $r\in \mathcal{R}$ for Householder projection axes\\
$T_r$  & $\mathbb{R}^{d \times m}$  & Param. of $r\in \mathcal{R}$ for Householder projection scalars\\

\bottomrule
\end{tabular}
\end{small}
\end{center}
\vskip -0.1in
\end{table}


\section{Proofs of Theorem \ref{thm-1}}
\label{proof}
\subsection{Proof of Lemma \ref{auxi}}
In order to prove Theorem \ref{thm-1}, we first prove an auxiliary Lemma \ref{auxi}.
\begin{lemma}
\label{auxi}
Any $k \times k$ orthogonal matrix $Q$ can be decomposed into the product of $k-1$ or $k$ Householder matrices.
\end{lemma}
\begin{proof}
From the Householder QR decomposition \cite{householder1958unitary}, we can upper triangularize any full-rank matrix $W \in \mathbb{R}^{k \times k}$ by using $k-1$ Householder reflections, i.e.,
\begin{equation}
    H(u_{k-1})H(u_{k-2})\cdots H(u_1)W=R,\nonumber
\end{equation}
where $R\in \mathbb{R}^{k \times k}$ is an upper triangular matrix and its first $n-1$ diagonal elements are all positive. 

When Household QR decomposition is performed on an orthogonal matrix $Q$, we can get: 
\begin{equation}
    H(u_{k-1})H(u_{k-2})\cdots H(u_1)Q=R.\nonumber
\end{equation}
Note that $R$ here is both upper triangular and orthogonal (i.e., $RR^T=I$) since it is a product of $k$ orthogonal matrices. 
It establishes that $R$ is a diagonal matrix, where the first $k-1$ diagonal entries are equal to $+1$ and the last diagonal entry is either +1 or -1. 


If the last diagonal entry of $R$ is equal to $+1$, we have
\begin{equation}
    H(u_{k-1})H(u_{k-2})\cdots H(u_1)Q=I.\nonumber
\end{equation}
As each Householder matrix $H(u_i)$ is its own inverse, we obtain that
\begin{equation}
\label{k-1}
    Q=H(u_1)\cdots H(u_{k-1}).
\end{equation}

If the last diagonal entry of $R$ is equal to $-1$, we can set $u_k=e_k=(0,\ldots,0,1)^\top \in \mathbb{R}^{k}$ and consequently get 
\begin{equation}
    H(u_k)R=H(u_k)H(u_{k-1})\cdots H(u_1)Q=I.\nonumber
\end{equation}
Since $H(u_i)$ is its own inverse, we also obtain that
\begin{equation}
\label{k}
    Q=H(u_1)\cdots H(u_{k-1})H(u_{k}).
\end{equation}
From Equation (\ref{k-1}) and (\ref{k}), we can see that any $k \times k$ orthogonal matrix can be decomposed into the product of $k-1$ or $k$ Householder matrices.
\end{proof}
\subsection{Proof of Theorem \ref{thm-1}}
\begin{proof}
We first prove that when $n=\lfloor \frac{k}{2} \rfloor$, the image of ${\rm Rot\mbox{-}H}$ is a subset of $\boldsymbol{\rm SO}(k)$, i.e., ${\rm Rot\mbox{-}H}(U)\subset \boldsymbol{\rm SO}(k)$.
Note that each Householder matrix is symmetric and orthogonal and its determinant is $-1$.
Therefore, the product of $2n$ Householder matrices is an orthogonal matrix with determinant $+1$, i.e., a rotation matrix~\cite{artin2016geometric}, which means ${\rm Rot\mbox{-}H}(U)\subset \boldsymbol{\rm SO}(k)$.

Then we also prove that its converse is also valid, i.e., any $k \times k$ rotation matrix can be expressed as the product of $2\lfloor \frac{k}{2} \rfloor$ Householder matrices $H(u_{i})$.
Note that a rotation matrix $\widetilde{Q}$ is a special orthogonal matrix with determinant $+1$~\cite{artin2016geometric}, i.e. $\operatorname{det}(\widetilde{Q})=+1$, and thus  $\widetilde{Q}$ can be decomposed into the product of $k-1$ or $k$ Householder matrices based on  Lemma \ref{auxi}. Moreover, since $\operatorname{det}(H(u_i))=-1$ and the determinant of a product of matrices is the product of their determinants, we can naturally derive that any $k \times k$ rotation matrix can be decomposed into the product of $2\lfloor \frac{k}{2} \rfloor$ Householder matrices, i.e., $\boldsymbol{\rm SO}(k)\subset {\rm Rot\mbox{-}H}(U)$. All in all, we have ${\rm Rot\mbox{-}H}(U)=\boldsymbol{\rm SO}(k)$.
\end{proof}

\section{Definitions}
\label{definition}

\begin{definition}
A relation $r$ is \textbf{symmetric} (\textbf{antisymmetric}) if $\forall{x, y}$
\begin{equation}
    r(x,y)\Rightarrow r(y,x) \quad (r(x,y)\Rightarrow \neg r(y,x)). \nonumber
\end{equation}
A clause with such form is a \textbf{symmetry} (\textbf{antisymmetry}) pattern.
\end{definition}
 
\begin{definition}
A relation $r_1$ is \textbf{inverse} to relation $r_2$ if $\forall{x, y}$
\begin{equation}
    r_2(x,y)\Rightarrow r_1(y,x) \nonumber.               
\end{equation}
A clause with such form is an \textbf{inversion} pattern.
\end{definition}

\begin{definition}
A relation $r_1$ is \textbf{composed} of relation $r_2$ and relation $r_3$ if $\forall{x, y, z}$
\begin{equation}
    r_2(x,y)\land r_3(y,z)\Rightarrow r_1(x,z). \nonumber
\end{equation}
A clause with such form is a \textbf{composition} pattern.
\end{definition}

Following \cite{TransE}, there are four relation mapping properties:

\begin{definition}
A relation $r$ is \textbf{1-to-1} if a $head$ can appear with at most one $tail$.
\end{definition}

\begin{definition}
A relation $r$ is \textbf{1-to-N} if a $head$ can appear with many $tails$.
\end{definition}

\begin{definition}
A relation $r$ is \textbf{N-to-1} if many $heads$ can appear with the same $tail$.
\end{definition}

\begin{definition}
A relation $r$ is \textbf{N-to-N} if many $heads$ can appear with many $tails$.
\end{definition}

\section{Proofs of Claims}
We denote the $r$-specific Householder rotation matrix and Householder projection matrices as $\widetilde{Q}_r$ and $\{W_{r,1}, W_{r,2}\}$ respectively:
\begin{equation}
    \begin{split}
\widetilde{Q}_r&={\rm Rot\mbox{-}H}({U_r[i]}),\\
W_{r,1}&={\rm Pro\mbox{-}H}(P_{r,1}[i], T_{r,1}[i]),\\
W_{r,2}&={\rm Pro\mbox{-}H}(P_{r,2}[i], T_{r,2}[i]). \nonumber
    \end{split}
\end{equation}
For simplicity, we also omit the row indices $[i]$ of entity representations in the following proofs. 
\subsection{Proof of Claim \ref{claim-1}}
\label{proof-claim-1}
\begin{proof}
if $r(x,y)$ and $r(y,x)$ hold, we have
\begin{equation}
S_y=\widetilde{Q}_r S_x \land S_x=\widetilde{Q}_r S_y \Rightarrow \widetilde{Q}_r\widetilde{Q}_r=I\nonumber
\end{equation}
Otherwise, if $r(x,y)$ and $\neg r(y,x)$ hold, we have
\begin{equation}
S_y=\widetilde{Q}_r S_x \land S_x \neq \widetilde{Q}_r S_y \Rightarrow \widetilde{Q}_r\widetilde{Q}_r\neq I\nonumber
\end{equation}
\end{proof}

\subsection{Proof of Claim \ref{claim-2}}
\label{proof-claim-2}
\begin{proof}
if $r_1(x,y)$ and $r_2(y,x)$ hold, we have
\begin{equation}
S_y=\widetilde{Q}_{r_1} S_x \land S_x=\widetilde{Q}_{r_2} S_y \Rightarrow \widetilde{Q}_{r_1}=\widetilde{Q}_{r_2}^{T}\nonumber
\end{equation}
\end{proof}

\subsection{Proof of Claim \ref{claim-3}}
\label{proof-claim-3}
\begin{proof}
if $r_1(x,z), r_2(x,y)$ and $r_3(y,z)$ hold, we have
\begin{equation}
S_z=\widetilde{Q}_{r_1}S_x \land S_y=\widetilde{Q}_{r_2}S_x \land S_z=\widetilde{Q}_{r_3}S_y \Rightarrow \widetilde{Q}_{r_1}=\widetilde{Q}_{r_3}\widetilde{Q}_{r_2}\nonumber
\end{equation}
\end{proof}

\subsection{Proof of Claim \ref{claim-4}}
\label{proof-claim-4}
\begin{proof}
if $r(x,y)$ and $r(y,x)$ hold, we have
\begin{equation}
\begin{split}
W_{r,2}S_y=\widetilde{Q}_r W_{r,1} S_x \land W_{r,2}S_x=\widetilde{Q}_r W_{r,1}S_y\\
\Rightarrow (W_{r,2}^{-1}\widetilde{Q}_r W_{r,1})(W_{r,2}^{-1}\widetilde{Q}_r W_{r,1})=I\nonumber
\end{split}
\end{equation}
Otherwise, if $r(x,y)$ and $\neg r(y,x)$ hold, we have
\begin{equation}
\begin{split}
W_{r,2}S_y=\widetilde{Q}_r W_{r,1} S_x \land W_{r,2}S_x \neq \widetilde{Q}_r W_{r,1}S_y\\
\Rightarrow (W_{r,2}^{-1}\widetilde{Q}_r W_{r,1})(W_{r,2}^{-1}\widetilde{Q}_r W_{r,1}) \neq I\nonumber
\end{split}
\end{equation}
\end{proof}

\subsection{Proof of Claim \ref{claim-5}}
\label{proof-claim-5}
\begin{proof}
if $r_1(x,y)$ and $r_2(y,x)$ hold, we have
\begin{equation}
\begin{split}
W_{r_1,2}S_y=\widetilde{Q}_{r_1} W_{r_1,1} S_x \land W_{r_2,2}S_x=\widetilde{Q}_{r_2} W_{r_2,1}S_y\\
\Rightarrow W_{r_1,2}^{-1}\widetilde{Q}_{r_1} W_{r_1,1}=(W_{r_2,2}^{-1}\widetilde{Q}_{r_2} W_{r_2,1})^{-1}\nonumber
\end{split}
\end{equation}
\end{proof}

\subsection{Proof of Claim \ref{claim-6}}
\label{proof-claim-6}
\begin{proof}
if $r_1(x,z), r_2(x,y)$ and $r_3(y,z)$ hold, we have
\begin{equation}
\begin{split}
W_{r_1,2}S_z&=\widetilde{Q}_{r_1} W_{r_1,1}S_x\\
\land W_{r_2,2}S_y&=\widetilde{Q}_{r_2} W_{r_2,1}S_x\\
\land W_{r_3,2}S_z&=\widetilde{Q}_{r_3} W_{r_3,1}S_y\\
\Rightarrow W_{r_1,2}^{-1}\widetilde{Q}_{r_1} W_{r_1,1}&=(W_{r_3,2}^{-1}\widetilde{Q}_{r_3} W_{r_3,1})(W_{r_2,2}^{-1}\widetilde{Q}_{r_2} W_{r_2,1})\nonumber
\end{split}
\end{equation}
\end{proof}


\subsection{Proof of Claim \ref{claim-8}}
\label{proof-claim-8}
In order to model sophisticated RMPs, we expect to tackle the challenge of indistinguishable representations with Householder projections as mentioned in Section \ref{subsec: rotation}. 
For the N-to-1 relations, here we take a 2-to-1 relation $r$ as an example with two triples $(h_1,r,t)$ and $(h_2,r,t)$. Householder projections can adjust the relative distance between entity $h_1$ and $h_{2}$ according to relation $r$. 
Formally, the original distance between $h_1$ and $h_{2}$ is defined as:  $s= \Vert S_{h_1}-S_{h_2}\Vert_2$. After  applying a modified Householder matrix, the relative  distance between the projected representations is: 
\begin{equation}
    \begin{split}
        \hat{s}^2 = \Vert S_{h_1,r}-S_{h_2,r}\Vert_2^2 = s^2 + (\tau_r^2 - 2\tau_r)s^2\cos^2\theta_{s,p_r}. \nonumber
    \end{split}
\end{equation}
It is clear that the learnable $\tau_r$ determines the increase or decrease of the relative distance:
(1) when $0<\tau_r<2$, $\hat{s}\le s$; (2) when $\tau_r=0$ or $2$, $\hat{s}=s$; (3) when $\tau_r<0$ or $\tau>2$, $\hat{s}\ge s$. Moreover, the term $\cos\theta_{s,p_r}$ is determined by the relative positions between the entities and the projection axis $p_r$.
This reveals that the
Householder projections can adaptively change the relative distance between entities based on their positions.
With such projections, one can obtain similar $r$-specific representations $S_{h_1,r}$ and $S_{h_2,r}$ for $h_1$ and $h_2$ 
, while the original representations $S_{h_{1}}$ and $S_{h_{2}}$ can be still distinguishable.
The same is also true for 1-to-N relations.

\section{Discussion on Other Relation Patterns}
\label{discussion}
\subsection{Multiplicity}
The multiplicity pattern has been investigated in DualE~\cite{cao2021dual}. Formally, it has the following definition:
\begin{definition}
Relation $r_1, r_2, \ldots, r_N$ are multiple if $\forall{i\in \{1,\ldots,N\}}$, $(h,r_i,t)$ can hold in KGs simultaneously. A clause with such form is a multiplicity pattern.
\end{definition}
DualE utilizes dual quaternions to represent each relation as a 3-dimensional rotation followed by a translation. It proves that the combination of rotations and translations can model multiple relations, since for any given rotation applied to the head entity $h$, there is always a corresponding translation to transform the rotated head entity to the tail entity $t$. 

In our proposed HousE, the relational Householder projections can be regarded as a special translation along the projection axes. 
Thus, HousE is similar to DualE in terms of multiplicity modeling capacity. 
What's more, as shown in Section \ref{translation plus}, our proposal can also easily integrate translations to achieve better performance. 
Geometrically, DualE can be viewed as a special case of HousE$^+$ with 3-dimensional rotations. 

\subsection{Transitivity.}
Rot-Pro~\cite{Rot-Pro} focuses on modeling the transitivity pattern, which is formally defined as:
\begin{definition}
A relation $r$ is transitive if for any instances $(e_1,r,e_2)$ and $(e_2,r,e_3)$ of relation $r$, $(e_1,r,e_3)$ is also an instance of $r$. A clause with such form is a transitivity pattern.
\end{definition}
Rot-Pro theoretically shows that the transitive relations can be modeled with a special orthogonal projections, which is designed to project the points onto the rotated axes. This kind of projections can be viewed as a 2-dimensional case of TransH's projections. 

HousE can be reduced to Rot-Pro if we set the rotation dimension to 2 and the projection scalars to 1. 
However, in our opinion, such projections may not be the optimal way to handle transitivity. As shown in~\cite{Rot-Pro}, Rot-Pro tends to project the entities under the transitive relation to a same point and the phase of relational rotation tends to be $2n\pi (n=0,1,2,\ldots)$. We can see that such solution is a subset of the solution of modeling symmetric relations, which means that the modeled transitive relations must be symmetric and the antisymmetric transitive relations are ignored. Therefore, how to comprehensively model the transitive relations is still a challenging problem, and we will take this as the future work.

\section{Datasets}
\label{datasets app}
\begin{table}[h]
\caption{Statistics of five standard benchmarks.}
\label{datasets}
\vspace{3mm}
\begin{center}
\begin{small}
\resizebox{0.55\columnwidth}{!}{\begin{tabular}{cccccc}
\toprule
Dataset & \#entity & \#relation & \#training & \#validation & \#test \\
\midrule
WN18    & 40,943 & 18 & 141,442 & 5,000 &  5,000 \\
FB15k    & 14,951 & 1,345 & 483,142 & 50,000 & 59,071 \\
WN18RR    & 40,943 & 11 & 86,835 & 3,034 & 3,134 \\
FB15k-237    & 14,541 & 237 & 272,115 & 17,535 & 20,466 \\
YAGO3-10    & 123,182 & 37 & 1,079,040 & 5,000 & 5,000 \\
\bottomrule
\end{tabular}}
\end{small}
\end{center}
\vskip -0.1in
\end{table}
Table \ref{datasets} summarizes the detailed statistics of five benchmark datasets:

WN18~\cite{TransE} is extracted from WordNet~\cite{miller1995wordnet}, a database featuring lexical relations between words.

FB15k~\cite{TransE} contains relation triples from Freebase~\cite{bollacker2008freebase}, a large-scale knowledge graph containing general knowledge facts. The main relation patterns in WN18 and FB15k are symmetry, antisymmetry and inversion. 

The WN18RR~\cite{dettmers2018convolutional} and FB15k-237~\cite{toutanova2015observed} datasets are subsets of WN18 and FB15k respectively with inverse relations removed. The key of link prediction on WN18RR and FB15k-237 boils down to model and infer the symmetry, antisymmetry and composition patterns. 

YAGO3-10 is a subset of YAGO3~\cite{mahdisoltani2014yago3}, containing 123,182 entities and 37 relations. Most of the triples in YAGO3-10 are descriptive attributes of people, such as citizenship, gender, profession and marital status. 


\section{Implementation details}
\label{Imp details}
Table \ref{params} shows the amount of parameters used in our models and several recent competitive baselines: RotatE, Rotate3D, QuatE and DualE. To ensure fair comparisons, we set the smaller embedding size $d$ to represent each entity and relation in HousE-r and HousE, so that the total number of parameters is similar to other baselines. Specifically, we fix the number of parameters $d \times k$ to represent a single entity as 1000, 1200, 800, 600, 1000 on WN18, FB15k, WN18RR, FB15k-237 and YAGO3-10, respectively. 
Hyperparameter $d$ denotes the embedding size and $k$ is the rotation dimension. 
The larger rotation dimension $k$ leads to the smaller embedding size $d$. 

From Table \ref{params}, one can see that our proposed models have similar numbers of parameters compared to the baselines. 
The only exception is QuatE on WN18RR and FB15k-237.  
We have tried to increase the number of parameters of QuatE by enlarging the embedding size $d$ on these two datasets, while carefully tuning hyperparameters simultaneously.
Unfortunately, the performance of QuatE drops with more free parameters.  
Thus, to ensure the fairness of performance comparison, we report the parameter numbers of QuatE with the best link prediction results.

\begin{table}[h]
\vspace{-2mm}
\caption{Number of free parameters comparison. The results of baselines are taken from the original papers. } 
\label{params}
\vskip 0.1in
\begin{center}
\begin{small}
\resizebox{0.6\columnwidth}{!}{
\begin{tabular}{ccccccc}
\toprule
Model  & RotatE & Rotate3D & QuatE  & DualE  & HousE-r  & HousE  \\
\cmidrule(r){1-5} \cmidrule(r){6-7}
WN18 & 40.95M & 122.90M & 49.15M & 65.53M & 40.88M & 41.03M \\
FB15k   & 31.25M & 50.23M & 26.08M & 26.08M & 24.40M & 27.63M \\
WN18RR     & 40.95M & 61.44M & 16.38M & 32.76M & 32.57M & 32.84M \\
FB15k-237    & 29.32M & 44.57M & 5.82M  & 11.64M & 12.13M & 13.36M \\
YAGO3-10    & 123.18M & - & -  & - & 122.91M  & 122.99M  \\
\bottomrule
\end{tabular}}
\end{small}
\end{center}
\end{table}

Table \ref{time} shows the convergence time required for the model training on five datasets. 
RotatE is the simplest rotation-based model with the linear time complexity, which is selected as the baseline.   
Compared to RotatE, our proposed HousE-r and HousE cost comparable or even less training time on these datasets by using the efficient computation in Equation (\ref{iter-eff}). 
Combined with the link prediction results in Table \ref{WN18-and-FB15k} and \ref{WN18RR-and-FB15k237}, one can see that our proposal is capable of improving model effectiveness without sacrificing the efficiency.
\begin{table}[h]
\vspace{-2mm}
\caption{Training time of RotatE and our proposal on five datasets.}
\label{time}
\vskip 0.1in
\begin{center}
\begin{small}
\begin{tabular}{cccccc}
\toprule
Model   & WN18RR & FB15k-237 & WN18 & FB15k & YAGO3-10 \\
\midrule
RotatE  & 4h         &  6h           &  4h        &  9h         & 10h         \\
HousE-r & 1.5h       &  3h           & 3h        & 8h         & 11h         \\
HousE   & 1.5h       & 5h           & 3h        &  9h         &  13h      \\
\bottomrule
\end{tabular}
\end{small}
\end{center}
\end{table}

We use Adam~\cite{kingma2014adam} as the optimizer and fine-tune the hyperparameters on the validation dataset. 
The hyperparameters are tuned by the random search~\cite{bergstra2012random}, including batch size $b$, self-adversarial sampling temperature $\alpha$, fixed margin $\gamma$, learning rate $lr$, rotation dimension $k$, number of modified Householder reflections $m$ for Householder projections, and regularization coefficient $\lambda$. The hyper-parameter search space is shown in Table \ref{Hyper_search}.
\begin{table}[h]
\vspace{-2mm}
\caption{Hyperparameter search space.}
\label{Hyper_search}
\vskip 0.1in
\begin{center}
\begin{small}
\begin{tabular}{ccc}
\toprule
Hyperparameter  & Search Space & Type  \\
\midrule
$b$  & $\{500, 800, 1000, 1500, 2000\}$ & Choice\\
$\alpha$  & $[0.5, 2.0]$ & Range\\
$\gamma$  & $\{5, 7, 9, 10, 11, 16, 20, 24, 28\}$ & Choice\\
$lr$  & $[0.0001, 0.003]$ & Range\\
$k$ & $\{2,4,8,12,16,20,25,30\}$ & Choice \\
$m$  & $\{1,2,3,4,6,8\}$ & Choice \\
$\lambda$  & $[0, 0.3]$ & Range\\
\bottomrule
\end{tabular}
\end{small}
\end{center}
\end{table}

\end{document}